\documentclass[letterpaper, 10 pt, conference]{ieeeconf}  

\IEEEoverridecommandlockouts                              

\usepackage{graphics} 
\usepackage{amsmath} 
\usepackage{microtype}

\usepackage{amsthm}
\newtheorem{definition}{Definition}
\newtheorem{assumption}{Assumption}
\newtheorem{theorem}{Theorem}

\usepackage{hyperref}
\usepackage{url}
\usepackage{graphicx}
\usepackage{xcolor}
\usepackage[final]{changes}
\setdeletedmarkup{}

\newcommand{\morgan}[1]{\textcolor{violet}{{[Morgan: #1]}}}

\title{\LARGE \bf
Reward as Observation: Learning Reward-Based Policies for Rapid Adaptation
}

\author{Morgan Byrd$^{1*}$, Jacob Blevins$^{1}$, Maks Sorokin$^{1}$, Robert Wright$^{2}$, and Sehoon Ha$^{1}$
\thanks{$^{1}$Georgia Institute of Technology, Atlanta, GA, 30308, USA}%
\thanks{$^{2}$Georgia Tech Research Institute, Atlanta, GA, 30308, USA}%
\thanks{*Correspondence to abyrd45@gatech.edu}
}

\begin{document}

\maketitle
\thispagestyle{empty}
\pagestyle{empty}

\begin{abstract}
This paper explores a reward-based policy to achieve zero-shot transfer between source and target environments with completely different observation spaces. While humans can demonstrate impressive adaptation capabilities, deep neural network policies often struggle to adapt to a new environment and require a considerable amount of samples for successful transfer. Instead, we propose a novel reward-based policy only conditioned on rewards and actions, enabling zero-shot adaptation to new environments with completely different observations. We discuss the challenges and feasibility of a reward-based policy and then propose a practical algorithm for training. We demonstrate that a reward policy can be trained within three different environments, Pointmass, Cartpole, and 2D Car Racing, and transferred to completely different observations, such as different color palettes or 3D rendering, or Stretch robot navigation in Habitat-Sim, in a zero-shot manner. We also demonstrate that a reward-based policy can further guide the training of an observation-based policy in the target environment.
\end{abstract}

\vspace{-0.2em}

\section{Introduction}
Humans readily adapt when the appearance of a familiar task changes. A gamer who masters a racing game can play immediately after a graphics engine upgrade, because the underlying rules and controls remain the same. Similarly, a robot trained to navigate a corridor should not need to relearn from scratch when its camera is replaced with a higher-resolution, wider Field-of-View or upgraded from RGB to RGB-D. In both cases, the core task structure, including dynamics, controls, and reward, is preserved; only the observation representation changes. However, typical deep RL policies treat such observation shifts as entirely new problems, requiring full retraining and discarding all previously learned task knowledge, because standard policies are tightly coupled to the specific observation space used during training.

We consider the problem of transferring a policy between two Markov Decision Process (MDP) environments that share the same underlying task structure, including state space, action space, transition dynamics, and reward function, but differ entirely in their observation space. Most prior transfer work in robotics focuses on transition dynamics shift, as in sim-to-real transfer \cite{kumar2021rma, liang2024rma, byrd-2025, doi:10.1126/scirobotics.abk2822}. Instead, we address the complementary case: fixed dynamics, but a completely different observation space. We impose no constraint on the relationship between source and target observations; they may differ arbitrarily in modality, dimensionality, or structure. We intentionally fix the dynamics to isolate the observation shift problem, and we expect our approach to compose naturally with existing methods that handle transition shifts.

Our key insight is that the reward signal directly encodes the task structure and remains consistent across both MDPs, and therefore serves as a natural bridge to enable rapid transfer between environments with different observations. Let us consider a driving task: if the reward stays high, the agent should continue its current behavior; if the reward drops, a steering correction is needed. This intuition holds whether the agent observes a 2D top-down view, a 3D camera feed, or the onboard camera of a Stretch mobile robot in a scanned 3D environment. This suggests that a scalar reward signal alone may carry sufficient information to address the challenge of drastic observation shift, but only if it is sufficiently dense and informative. Under this condition, the reward signal has the potential to provide effective guidance in the target environment, regardless of how the observations change.

We propose a reward-based policy as a transferable teacher that bridges source and target environments under drastic observation shifts. Training such a policy is inherently challenging, as the reward signal is only a scalar projection of the underlying state space with limited information for control. To overcome this, we adopt an LSTM-based architecture that processes the temporal history of reward and action pairs, and incorporate a guidance loss from an observation-based expert policy to stabilize training. Once trained in the source environment, the reward-based policy is deployed in the target environment as a teacher, and we use DAgger to train an observation-based policy significantly faster than learning from scratch.

\begin{figure}[t!]
\begin{center}
\includegraphics[width=0.8\linewidth]{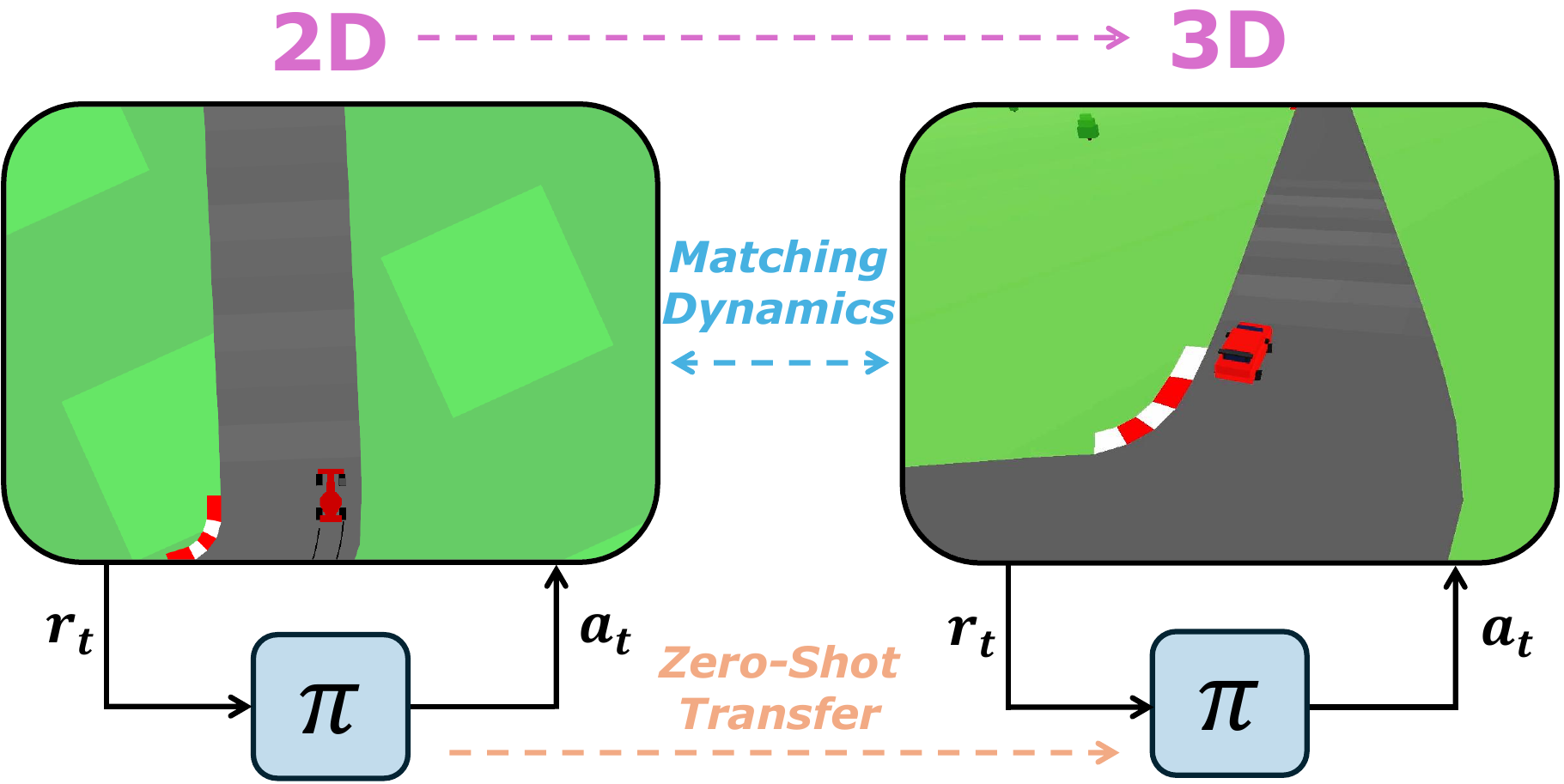}
\end{center}
\vspace{-1em}
\caption{A reward-based policy enables zero-shot transfer from the 2D Car Racing environment to the 3D Car Racing environment by taking reward and action histories as the only input.}
\label{fig:2d_3d_transfer}
\vspace{-1em}
\end{figure}

\begin{figure}[t!]
\begin{center}
\includegraphics[width=0.97\linewidth]{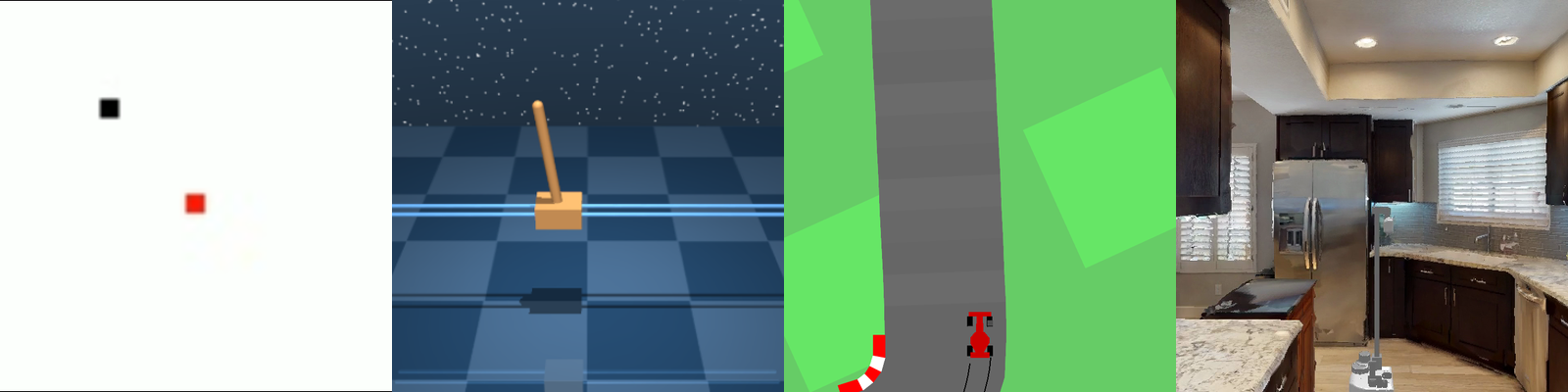}
\end{center}
\vspace{-1em}
\caption{Evaluation environments: Pointmass, Cartpole, Car Racing, and a Habitat-Sim environment for transfer testing.}
\label{fig:tasks}
\vspace{-1.2em}
\end{figure}

We demonstrate that a reward-based policy can serve as an effective transferable teacher across environments with arbitrary observation shifts. First, reward-based policies achieve $66$\% to $95$\% of state-based performance across three environments. Second, they transfer zero-shot under drastic observation shifts, including 2D-to-3D rendering changes and deployment on photorealistic Habitat-Sim \cite{puig2023habitat3} reconstructions of real building interiors with a Stretch mobile robot. Third, using the reward-based policy as a DAgger teacher enables substantially faster learning in the target environment. Finally, reward estimation proves far more tractable than state estimation under limited data.

The technical contributions of this paper are:
\begin{itemize}
    \item We define the reward-based MDP and discuss its theoretical properties, including feasibility conditions and an optimality gap bound (Appendix).
    \item We propose a practical implementation for training reward-based policies, and a DAgger-based algorithm that uses the reward-based policy as a zero-shot teacher to train an observation-based policy in the target environment.
    \item We demonstrate reward-based policy training across diverse environments and validate rapid transfer under drastic observation shifts, including 2D-to-3D transfer and zero-shot deployment on a Stretch mobile robot in photorealistic Habitat-Sim reconstructions of real building interiors.
\end{itemize}
\vspace{-0.55em}
\section{Related Work}

\noindent\textbf{Transfer via Expert Demonstrations.}
Transfer learning in RL leverages existing information for rapid adaptation in a new scenario, which has been approached by a wide range of algorithms. One possible approach is learning from demonstration \cite{NIPS1996_68d13cf2}, where expert demonstrations are available to guide the learning process. Offline methods can use these demonstrations for offline RL \cite{Ma2019, Yang2021, li2023mahalo} and for pretraining to learn value functions or policies \cite{Silver2016, stachowicz2023fastrlap}. Online methods use expert demonstrations to improve policy exploration during learning, and include policy iterations \cite{10.1007/978-3-662-44851-9_35, 10.5555/2832581.2832720}, policy gradients \cite{8463162, Kang2018}, and Q-learning approaches \cite{10.5555/2832581.2832716, 10.5555/3504035.3504429}. Our reward-based policy serves a similar role, as it automatically provides guidance in the target environment through the shared reward signal.

\noindent\textbf{Transfer via Representation Learning.}
Another way to transfer knowledge is representation learning, which establishes reusable representations in the action or state spaces that can be shared between source and target environments. Progressive networks \cite{Rusu2016} and PathNet \cite{Fernando2017} directly train common reusable representations and leverage them during policy training. There exists a line of works that split networks into general and reusable modules \cite{Andreas2017, 7989250, Gupta2017} and learn an invariant feature space in order to better transfer between tasks. \cite{Ammar2012, JMLR07-taylor} learn inter-task mappings to deal with state differences between tasks. 
Successor representations \cite{6795455} seek to separate the environment dynamics from the reward function. These have seen widespread use \cite{Barreto2018, 8206049, Konidaris2006} for transferring knowledge for different tasks with the same state and action spaces. \cite{ZhangAmy2018} does a similar approach of separating dynamics and reward for learning a better policy for transfer. In contrast, our method avoids the observation alignment problem by using the shared reward signal as a bridge.

\noindent\textbf{Transfer via Policy Distillation.}
The most relevant method of transfer for our work is to transfer at the policy level. A popular method to transfer involves distilling information from an expert policy.  This method uses supervised learning to match action distributions between the student and teacher policies, and is widely used for robotic applications operating under observation shift between student and teacher \cite{kumar2021rma, liang2024rma, byrd-2025, doi:10.1126/scirobotics.abk2822}. Policies can also be directly reused for updating policies in a transfer setting, as in \cite{Barreto2017, 10.1145/1160633.1160762, DBLP:conf/iclr/RajendranLKPR17}. Our method is most closely related to this line of work, but uses a reward-based policy trained in the source environment as the teacher, requiring no observation alignment.

\noindent\textbf{Reward-Conditioned Policies.}
Another very similar sounding work to ours is the reward-conditioned policy \cite{RCPKumar19}. However, this work trains standard observation-based policies with an additional reward conditioning that allows for using suboptimal trajectories as optimal supervision. In contrast, our method uses the reward signal to train a transferable teacher in the source environment, which accelerates learning in the target environment.

\section{Problem Definition: Transfer across Different Observation Spaces}
\label{sec:problem}

A Partially Observable Markov decision process (POMDP) is a popular tool for modeling a sequential decision problem. It is defined as a tuple $(S, O, A, R, T)$, where $S$ is the state space, $O$ is the observation space, $A$ is the action space, $T$ is the transition function that defines how states change given an action, and $R$ is the reward function. The policy $\pi: O \mapsto A$ takes action based on the given observation without knowing the underlying state $s \in S$, and we want to find an optimal policy that maximizes the discounted cumulative sum of the reward over time: $\sum_{t = 0}^{T} \gamma^{t} R(s_t, a_t)$.

Our scenario defines the transfer problem as training a reusable policy from a source to a target environment. We assume the source and target environments should have the same transition, reward, action, and state but may have very different observation spaces. The source environment is defined as $\mathcal{M}_{\text{src}} = (S, O_{\text{src}}, A, R, T)$ and the target environment as $\mathcal{M}_{\text{tgt}} = (S, O_{\text{tgt}}, A, R, T)$, where we impose no constraint on the relationship between $O_{\text{src}}$ and $O_{\text{tgt}}$. In general, a policy trained in $\mathcal{M}_{\text{src}}$ cannot be directly reused in $\mathcal{M}_{\text{tgt}}$, as the input dimensionality and structure of $O_{\text{src}}$ and $O_{\text{tgt}}$ may differ arbitrarily. Without specific assumptions on the relationship between the two observation spaces, the policy must be retrained from scratch in $\mathcal{M}_{\text{tgt}}$, which makes transfer challenging.

\section{Reward-based Policy}

\subsection{Reward-based MDP}

We define a reward-based MDP $\mathcal{M}^{\text{rwd}}_E = (S, O^{\text{rwd}}, A, R, T)$ of an original environment $E$ by replacing the original observation space $O$ with a reward-based observation $O^{\text{rwd}} = \{(r_{t-k}, a_{t-k})\}_{k=0}^{K-1}$, which is a history of the most recent $K$ reward and action pairs. Here, $K$ is a hyperparameter controlling the length of the history. The corresponding reward-based policy $\pi^{\text{rwd}}: O^{\text{rwd}} \mapsto A$ takes only the reward/action history as input, with no access to the original observations. This formulation requires a dense reward signal; without sufficient reward coverage over the state space, the policy cannot acquire enough information for effective control.

The key property of the reward-based MDP is that it is independent of the original observation space $O$. Since both $\mathcal{M}_{\text{src}}$ and $\mathcal{M}_{\text{tgt}}$ share the same reward and action spaces, $\mathcal{M}_{\text{src}}^{\text{rwd}}$ and $\mathcal{M}_{\text{tgt}}^{\text{rwd}}$ are identical. A policy trained in $\mathcal{M}_{\text{src}}^{\text{rwd}}$ can therefore be directly applied in $\mathcal{M}_{\text{tgt}}^{\text{rwd}}$ without any modification. In contrast, a standard observation-based policy trained in $\mathcal{M}_{\text{src}}$ cannot be reused in $\mathcal{M}_{\text{tgt}}$, as $O_{\text{src}}$ and $O_{\text{tgt}}$ may differ arbitrarily in dimensionality and structure.

\begin{table}[t]
\centering
\caption{Maximum reward (\% of optimal) for Pointmass as a function of the number of dimensions of state space.}
\begin{tabular}{lll}
\hline
\textbf{Dimension} & \textbf{State-based} & \textbf{Reward-based} \\ \hline
1 & 99.83 & 98.73 \\ \hline
2 & 99.60 & 97.11 \\ \hline
3 & 99.48 & 95.94 \\ \hline
4 & 99.34 & 88.12 \\ \hline
5 & 99.14 & 85.96 \\ \hline
6 & 99.05 & 57.18 \\ \hline
\end{tabular}
\label{table:nd_pointmass}
\vspace{-1.5em}
\end{table}

\vspace{-0.5em}

\subsection{Properties}

\subsubsection{Difficulty: Limited Observability}
Learning a policy from only reward information is more difficult than using a standard observation-based approach. This is because the reward is a scalar projection of the, in general, $D$-dimensional state space that governs the dynamics of the task. The observation representation for a policy is crucial~\cite{kim2021observation}, and this reduction in information makes the learning task much more challenging as the partial observability makes estimating the value and advantage functions more difficult. Exploration is also negatively impacted by the lower information feedback.

This difficulty is particularly pronounced in higher dimensions. The optimality gap of a reward-based policy is bounded linearly by the state reconstruction error $\epsilon$ (see Appendix). Because inferring state from scalar reward signals implicitly requires orthogonal probing steps (analogous to trilateration), the reconstruction error is expected to scale as $\mathcal{O}(c^D)$ for some $c > 1$. This exponential growth in $\epsilon$ with respect to $D$ offers a theoretical explanation for the sharp performance degradation observed as dimensionality increases.

\subsubsection{Feasibility: Low-dimensional Environments}
Despite these difficulties, reward-based policies remain feasible in some environments. Takens' Embedding Theorem \cite{takens1981detecting} provides a theoretical basis: a $D$-dimensional state can in principle be reconstructed from a scalar observation history of length $K \ge 2D+1$, given a sufficiently smooth observation function. This supports the use of a reward-action history as a sufficient input for policy learning, and in practice we find that reward-based policies can be trained with off-the-shelf RL algorithms.

To illustrate this, consider a simple 1D Pointmass environment where the reward is defined as a negative distance to the origin. A reward-based agent can estimate its position by checking how the reward changes in response to its actions. Even in $D-$dimensional Pointmass environments, an analytical algorithm similar to trilateration can determine the location from a sequence of reward and action pairs. However, as the dimension increases, the time required grows nearly exponentially.

Table~\ref{table:nd_pointmass} illustrates the feasibility and the dimensional scaling of reward-based policies. State-based policies maintain near-optimal performance regardless of dimension. Reward-based policies remain feasible at least up to 6D, which confirms that learning from reward history alone is achievable. However, with a fixed history length $K$, performance degrades sharply as $D$ increases, achieving only $57\%$ of the maximum reward in the $6$D environment. This is consistent with the $\mathcal{O}(c^D)$ error bound: when $K$ is fixed, the reconstructive capacity does not scale with $D$, and the optimality gap widens accordingly.

\subsubsection{Requirement: Dense Rewards}
Reward-based policy learning requires a dense reward function. Because the reward serves as the sole observation, it must provide an informative signal throughout the state space. A sparse reward, such as one given only at the goal, leaves the agent without any useful feedback for most of the trajectory, which makes learning infeasible in practice.

\subsubsection{Limitation: Suboptimal Exploration}
Reward-based policies also exhibit suboptimal exploration behavior. In a $2$D Pointmass environment, an observation-based agent with direct access to $x$ and $y$ can act optimally from the start. A reward-based agent, however, can only infer its distance to the goal from a single reward, and must take at least two information-gathering actions before it can localize itself and act optimally. This overhead is inherent to the reward-based formulation, and is expected to worsen in higher-dimensional or more complex environments.

\vspace{-0.25em}
\subsection{Training Reward-Based Policies}

\begin{figure*}
\vspace{1.5em}
\begin{center}
\includegraphics[width=0.8\linewidth]{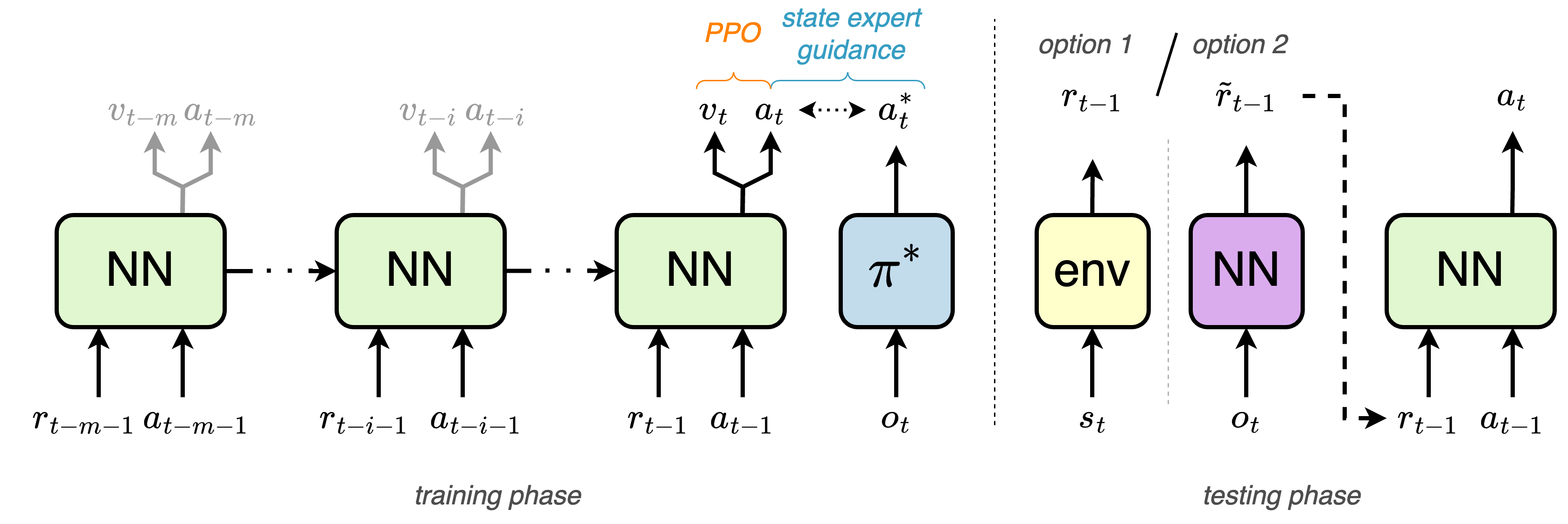}
\end{center}
\vspace{-1em}
\caption{Architecture. The reward-based policy $\pi_R$ is trained using expert state-based guidance from $\pi^{*}$ and standard training in RL with PPO.}
\label{fig:architecture}
\vspace{-1.2em}
\end{figure*}

Based on the analysis above, we identify two key components for training a reward-based policy effectively:

\begin{enumerate}
    \item Temporal history of reward/action pairs.
    \item Guidance from an observation-based expert policy.
\end{enumerate}

\noindent\textbf{Temporal History via LSTM.}
While Takens' Embedding Theorem suggests a fixed history length $K \ge 2D+1$ is sufficient in principle, selecting the right $K$ requires knowing $D$ in advance, which varies across environments. Instead, we use an LSTM~\cite{LSTM} to implicitly maintain a variable-length history, which avoids the need to manually tune $K$. Each reward/action pair is first processed by a small MLP for feature extraction, and the resulting features are passed to the LSTM. The LSTM output is then used as input to the policy and value networks, trained with PPO~\cite{1707.06347} using code from \cite{huang2022cleanrl}.

\noindent\textbf{Expert Guidance.}
Temporal history alone is not sufficient for effective learning in practice. Without additional signal, the policy struggles to explore productively under the limited observability of reward-only inputs. To address this, we add a supervised guidance loss that regresses the policy actions toward those of a pre-trained observation-based expert $\pi^*$. The full training loss is $L = L_{\text{PPO}} + L_{\text{guidance}}$, where $L_{\text{guidance}} = (a_t - a_t^*)^2$.


\subsection{Training Observation-Based Policies in Target Environments}
Once a reward-based policy $\pi^{\text{rwd}}$ is trained in the source environment $\mathcal{M}_{\text{src}}$, it can be deployed directly in the target environment $\mathcal{M}_{\text{tgt}}$, since both share the same reward and action spaces. We leverage this transferability to train an observation-based policy $\pi_{\text{tgt}}$ in $\mathcal{M}_{\text{tgt}}$ using DAgger~\cite{pmlr-v15-ross11a}, where $\pi^{\text{rwd}}$ serves as the teacher. At each iteration, $\pi_{\text{tgt}}$ collects trajectories in $\mathcal{M}_{\text{tgt}}$, and the reward-based policy provides action labels for the visited states. The resulting observation-based policy $\pi_{\text{tgt}}$ operates purely on target observations at inference time, without requiring access to rewards.


\section{Results}

\subsection{Reward-Based Policies Train Successfully Across Diverse Environments}

We evaluate on three source environments shown in Fig.~\ref{fig:tasks}, selected in order of increasing task complexity. The simplest environment tests the basic feasibility of reward-based learning, while the most complex evaluates transfer under a maximally challenging observation shift.

\textbf{Pointmass} is a 2D environment where the agent moves a point mass to a target location with kinematic actions $s_{t+1} = s_t + a_t$. The reward is the negative exponential of the distance to the goal, which provides a dense signal throughout the state space. This environment is intentionally simple and serves as a controlled testbed for reward-based policy learning.

\textbf{Cartpole} is the pole-balancing task from the DeepMind Control Suite~\cite{Tassa2018}. The agent applies continuous torque to balance a pole on a moving cart. The reward is proportional to the cosine of the pole angle, which remains dense and informative throughout the episode, as well as regularizing terms for torque, velocity, and positioning.

\textbf{Car Racing} uses the Gymnasium~\cite{towers2024gymnasium} 2D top-down environment as the source. The original reward is sparse, so we replace it with a dense formulation: $r = r_{\text{distance}} \cdot r_{\text{angle}}$, where $r_{\text{distance}} = \exp(-d)$ penalizes distance to the next waypoint and $r_{\text{angle}} = \cos(\theta)$ penalizes heading deviation from the track direction. To test the most extreme case of observation shift, we developed a custom 3D rendering engine that reproduces the same track layout and transition dynamics as the 2D source. This setup enables a direct 2D-to-3D transfer evaluation, where the observation space changes from a simple top-down view to a full 3D camera perspective, while all other aspects of the task remain identical.

Fig.~\ref{fig:learning_curves} shows learning curves across three random seeds, with state-based policies as the performance upper bound. Reward-based policies achieve $95\%$, $66\%$, and $70\%$ of state-based performance in Pointmass, Cartpole, and Car Racing, respectively. While reward-based policies do not fully reach the state-based upper bound, this gap is consistent with the theoretical bound in Theorem~\ref{thm:optimality_bound}. Since the amount of information discarded by the scalar reward signal grows with state dimensionality, the achieved performance is respectable given the fundamental partial observability. A Pointmass agent navigates to the goal after a few exploratory steps. 
In Car Racing, the agent occasionally deviates from the track centerline to acquire additional reward signal before correcting course. This level of performance is sufficient for the transfer experiments in the following sections, as demonstrated in the supplementary video.

\begin{figure*}[!t]
\vspace{1.5em}
\centering
\begin{tabular}{c c c}
     \includegraphics[width=0.29\linewidth]{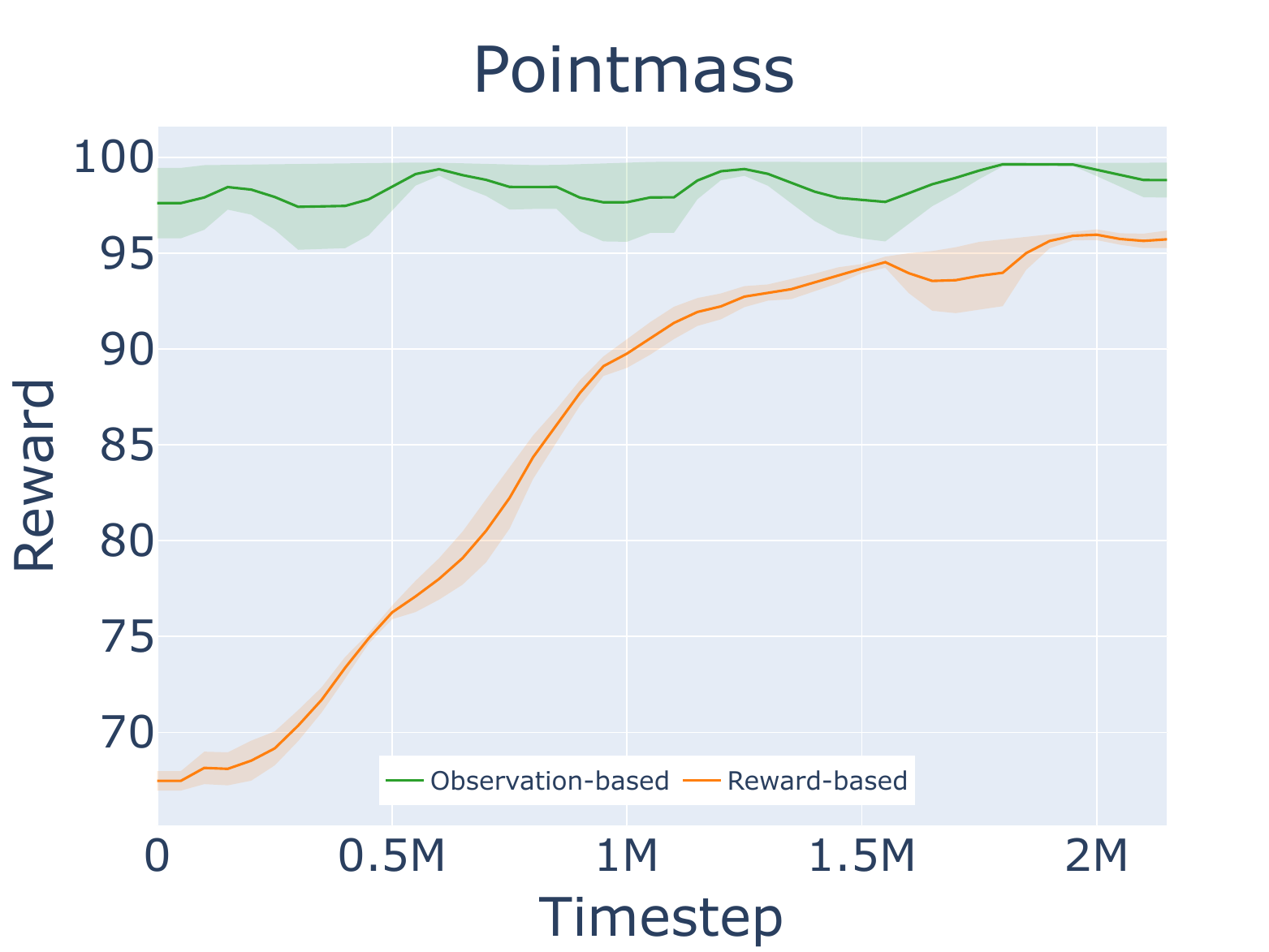} & \hspace{-2em}
     \includegraphics[width=0.29\linewidth]{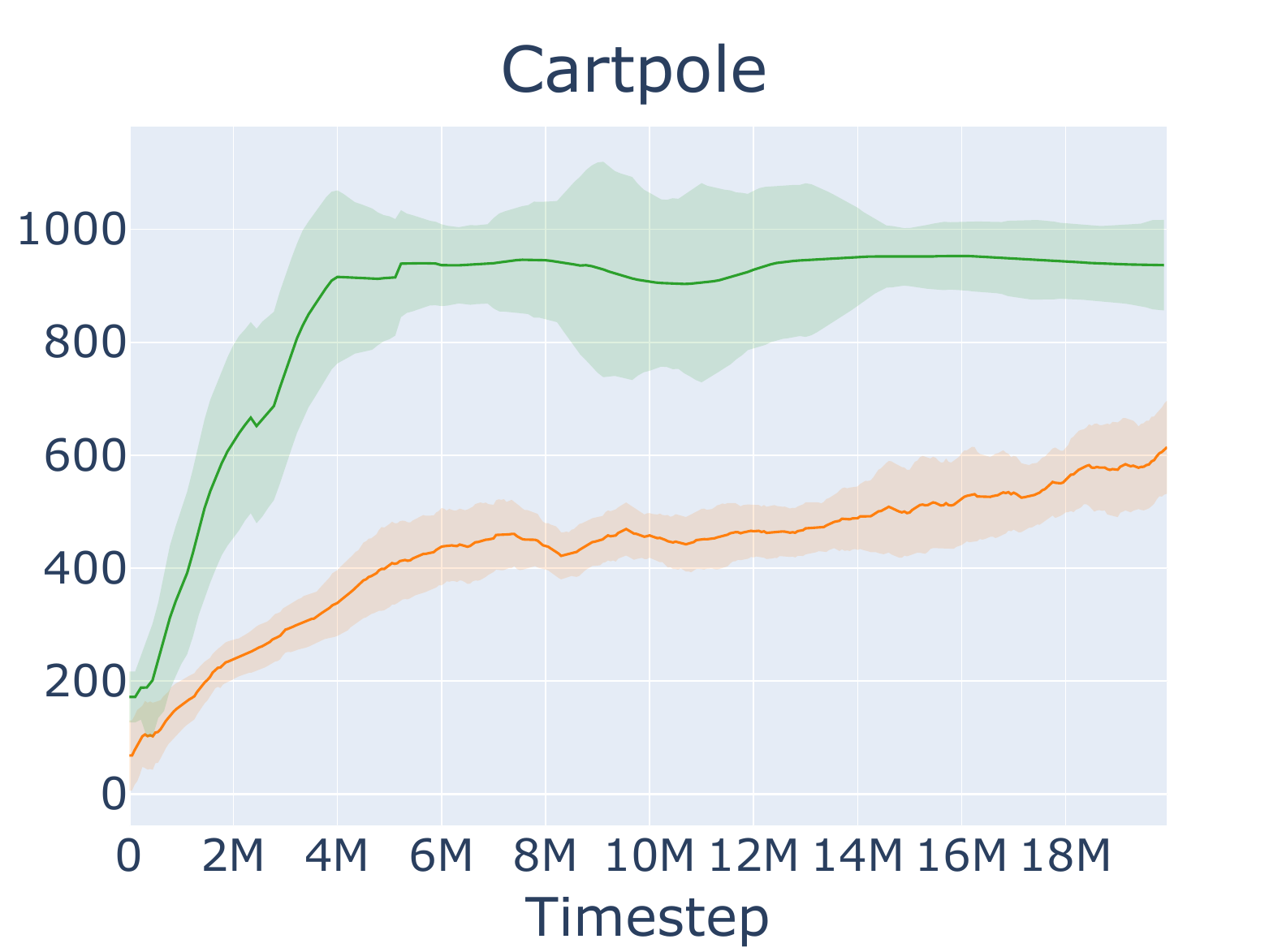} & \hspace{-2em}
     \includegraphics[width=0.29\linewidth]{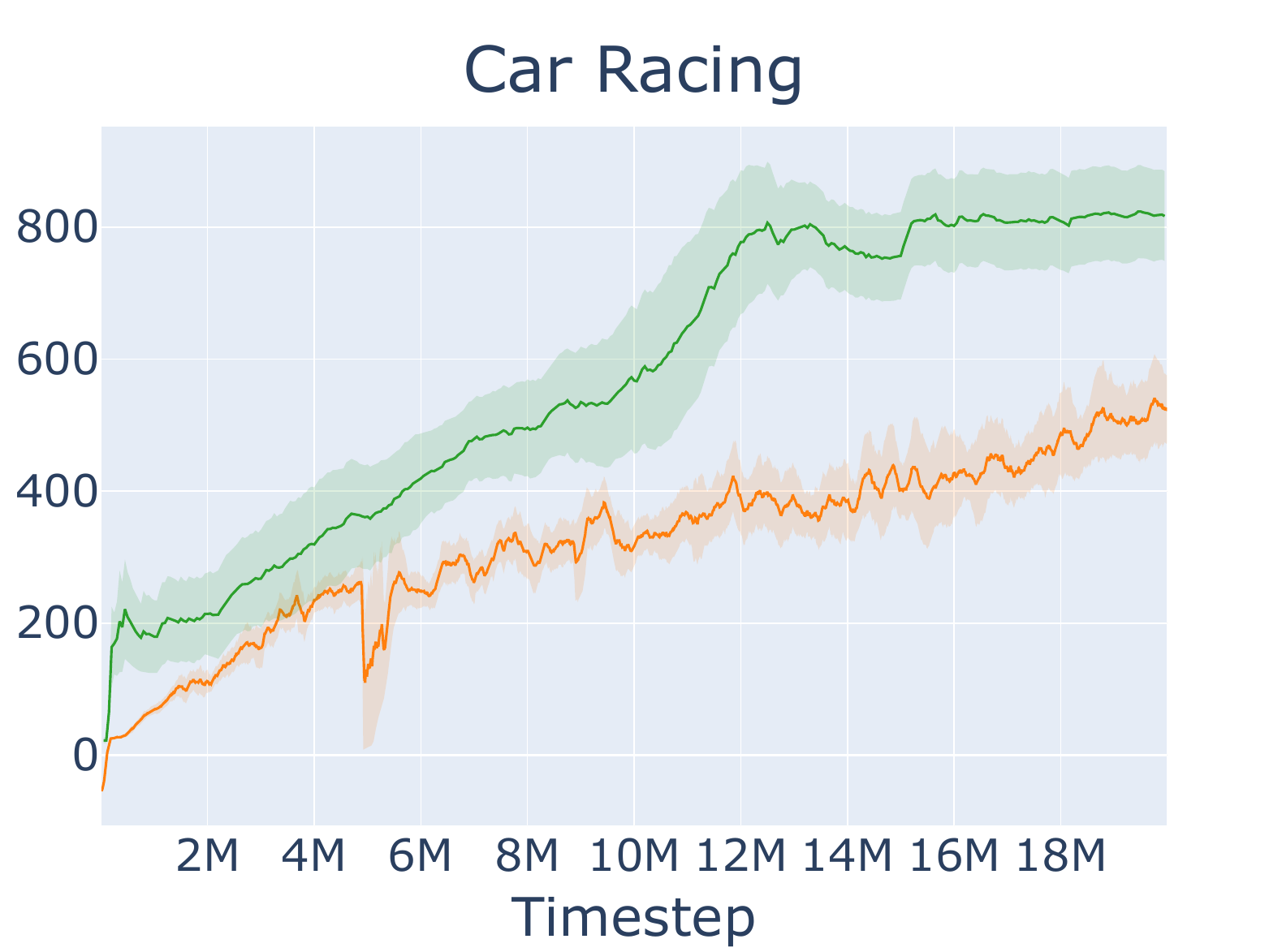}
\end{tabular}
\caption{Learning curves for reward-based and state-based policies across three environments. State-based policies serve as the performance upper bound.}
\label{fig:learning_curves}
\vspace{-1em}
\end{figure*}

\subsection{Reward-Based Policies Transfer Zero-Shot Under Drastic Observation Shifts}

Because a reward-based policy takes only reward and action as input, its behavior is independent of the observation representation. A policy trained in the source environment can be deployed in any target environment that shares the same reward and action spaces, without modification.

We first evaluate on a color palette swap, where the observation colors are changed while all other aspects of the environment remain the same. Fig.~\ref{fig:color_swap} shows example observation shifts across all three environments. Table~\ref{table:color_swap} shows that reward-based policies achieve nearly identical performance in both the original and palette-swapped environments. This confirms that observation changes have no effect on reward-based policy behavior.

\begin{figure}[t]
\begin{center}
\includegraphics[width=\linewidth]{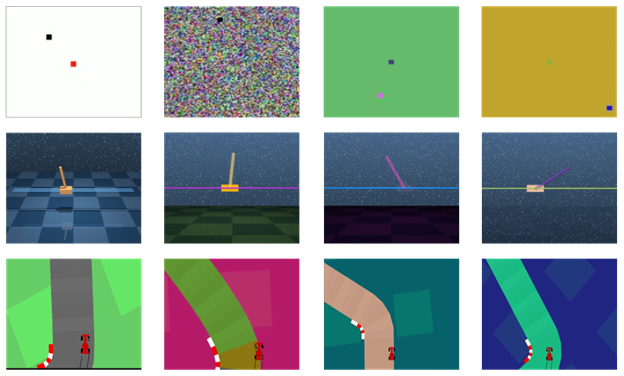}
\end{center}
\vspace{-0.75em}
\caption{A reward-based policy transfers to any observation variant in a zero-shot manner, including intentionally difficult color palette swaps.}
\label{fig:color_swap}
\vspace{-0.75em}
\end{figure}

\begin{table}[t]
\vspace{-0.5em}
\centering
\caption{Zero-shot transfer under color palette swaps. Mean reward value $\pm$ standard deviation over 50 trials. Performance is functionally identical in both observation domains.}
\begin{tabular}{lll}
\hline
\textbf{Task} & \textbf{Original Observation} & \textbf{Shifted Observation} \\ \hline
Pointmass & 97.06 $\pm$ 1.36 & 95.96 $\pm$ 2.47 \\ \hline
Cartpole & 758.74 $\pm$ 62.79 & 755.45 $\pm$ 52.63 \\ \hline
Car Racing & 775.37 $\pm$ 225.27 & 715.93 $\pm$ 264.28 \\ \hline
\end{tabular}
\label{table:color_swap}
\vspace{-1.1em}
\end{table}

We then evaluate more challenging shifts: transferring from the 2D Gymnasium Car Racing environment to both a 3D rendering of the same task (Fig.~\ref{fig:2d_3d_transfer}) and to photorealistic Habitat-Sim reconstructions of real building interiors driven by a Stretch mobile robot (Fig.~\ref{fig:habitat_motion_strip}). For the 3D rendering, we match the transition function to the 2D source to isolate the observation shift. The reward-based policy transfers directly and demonstrates optimal driving behavior indistinguishable from the 2D expert reference. For Habitat, the same checkpoint is used with no retraining on the navmesh-derived corridor tracks recorded in two real HM3D~\cite{ramakrishnan2021hm3d} building scans, succeeding in both scenes and producing sustained navigation despite the entirely different rendering pipeline and robot embodiment. As in Car Racing, the reward $r = r_{\text{distance}} \cdot r_{\text{angle}}$ is computed from the pose relative to the recorded waypoints, which we assume can be obtained via localization during real deployments. Neither transfer requires retraining or observation alignment, showing that the reward signal bridges arbitrarily different observation spaces.

\begin{figure*}[!t]
\begin{center}
\includegraphics[width=\linewidth]{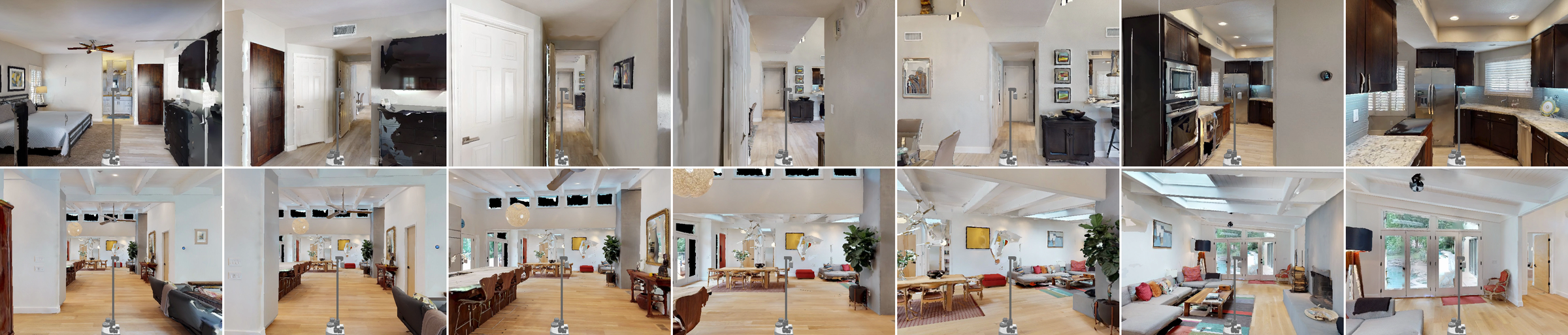}
\end{center}
\vspace{-0.5em}
\caption{Motion strip of the reward-based policy deployed zero-shot (no retraining) in Habitat-Sim on two real building scans (HM3D), driving a Stretch mobile robot. Top: scene 00801 (bedroom-to-kitchen corridor). Bottom: scene 00807 (open living/dining area).}
\label{fig:habitat_motion_strip}
\vspace{-1em}
\end{figure*}

\subsection{Reward-Based Policies Accelerate Observation-Based Policy Learning in the Target Environment via DAgger}

The zero-shot transfer results above establish that a reward-based policy operates in the target environment without modification. We now use this property to train an observation-based policy in the target environment, with the reward-based policy serving as a DAgger~\cite{pmlr-v15-ross11a} teacher.

We use the 3D Car Racing environment as the target. Training an observation-based policy from scratch in this setting is challenging due to high-dimensional pixel observations. We use a CNN-LSTM architecture that first processes pixel input with a CNN and passes features to an LSTM before the actor and critic networks. As shown in Fig.~\ref{fig:3d_learning_curve}, learning from scratch converges slowly; notably, this trend is visible even on a logarithmic timestep scale.

\begin{figure}[t]
\begin{center}
\includegraphics[width=0.9\linewidth]{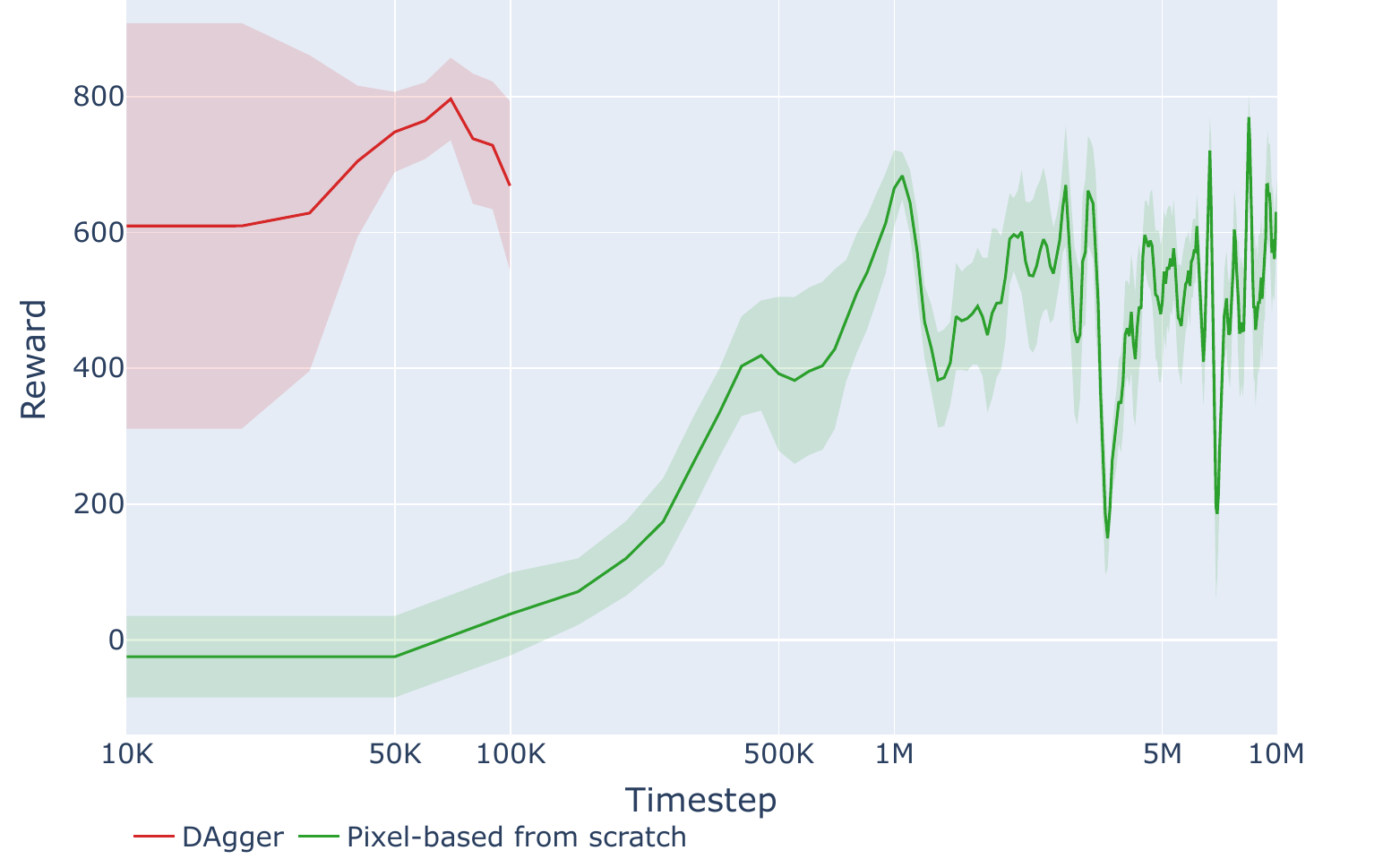}
\end{center}
\vspace{-1.em}
\caption{Learning curves for the 3D Car Racing environment on a logarithmic timestep scale. Learning from scratch (CNN-LSTM) is compared against DAgger distillation using the reward-based policy as teacher.}
\label{fig:3d_learning_curve}
\end{figure}

\begin{figure*}[!t]
\vspace{1.5em}
\begin{center}
\includegraphics[width=0.9\linewidth]{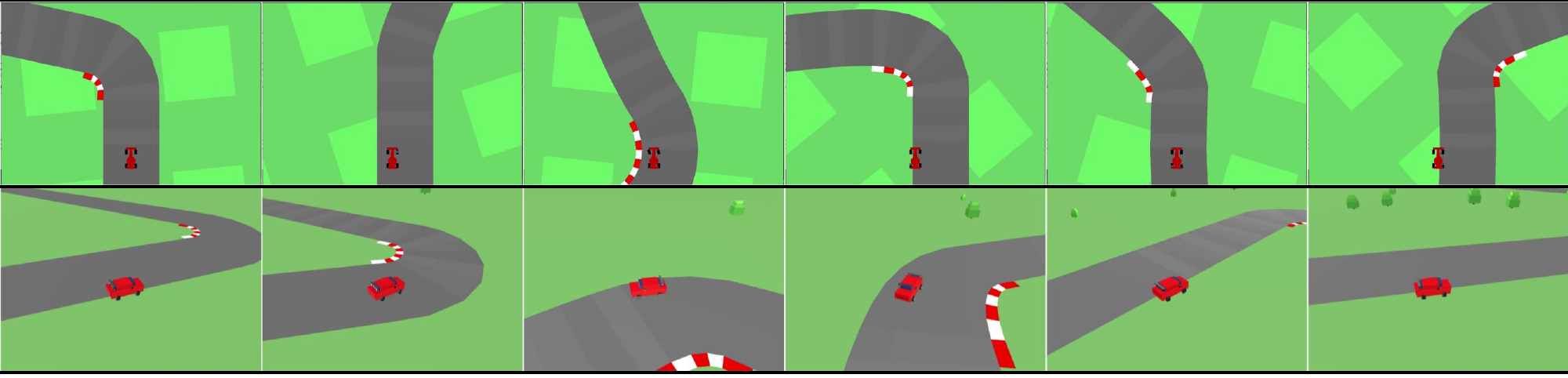}
\end{center}
\vspace{-0.5em}
\caption{Motion strip comparing the source 2D Car Racing environment (top) and the target 3D environment (bottom). The reward-based teacher enables the student policy to reproduce consistent driving behavior across both observation spaces.}
\label{fig:motion_strip}
\vspace{-1em}
\end{figure*}

Using the reward-based policy as a DAgger teacher, the student policy learns substantially faster and reaches near optimal performance in $70{,}000$ steps. Fig.~\ref{fig:motion_strip} shows a motion strip comparing the source 2D environment and the target 3D environment, demonstrating that the student policy reproduces consistent driving behavior across both observation spaces.

One may question whether the upfront training cost in the source environment offsets the benefit. Two considerations argue otherwise. First, a reward-based policy transfers to any target environment that shares the same reward and action spaces, so the training cost is amortized across multiple transfers. Second, per-step simulation cost often differs substantially between environments: source environments are typically lightweight and support massively parallel rollouts, while target environments rendered by commercial graphics engines can be substantially more costly per step. Concentrating training in the inexpensive source and transferring rapidly to the expensive target therefore yields a clear practical advantage in such settings.

\subsection{Reward Estimation Succeeds Where State Estimation Fails}

One may suggest a simpler alternative: learn a state estimator from target observations and apply the existing state-based policy directly. We compare this baseline against reward-based transfer using the DMC Cartpole task, where the target environment provides image observations.

Both methods collect $100{,}000$ samples from the target image domain and train their respective estimators using supervised learning. Table~\ref{table:estimation} shows the results. The state-based policy with a learned state estimator performs at the level of a random policy (189 vs. random baseline 192). The five-dimensional state is too difficult to recover accurately from limited data. The reward-based policy with a learned reward estimator maintains substantial performance (735), close to its source-domain score of 836.

\begin{table}[t]
\centering
\caption{A reward-based policy can be used even without a reward at inference, as estimating the reward from an image is easier than the full state, allowing the policy to be used directly from a small set of training. Means and standard deviations are reported over 100 trials with each policy.}
\begin{tabular}{lll}
\hline
\textbf{Policy} & \textbf{Original Domain} & \textbf{Image Domain} \\ \hline
State-based & 915.49 $\pm$ 2.24 & 189.41 $\pm$ 12.54 \\ \hline
Reward-based & 836.49 $\pm$ 4.46 & 735.83 $\pm$ 65.48\\ \hline
Random & \hspace{1cm} - & 192.78 $\pm$ 36.91 \\ \hline
\end{tabular}
\label{table:estimation}
\vspace{-0.4em}
\end{table}

This gap reflects a fundamental difference in estimation difficulty. A reward is a one-dimensional scalar and can be accurately recovered from limited data. A state is higher-dimensional and far more sensitive to estimation error under the same data budget. Reward-based transfer therefore reduces the transfer problem to estimating the simplest possible signal, while state-based transfer requires recovering full high-dimensional structure.

\subsection{Expert Guidance Improves Training in Complex Environments}

\begin{figure*}[!t]
\centering
\begin{tabular}{c c c}
     \includegraphics[width=0.3\linewidth, trim={0, 0, 0.5cm, 0}, clip]{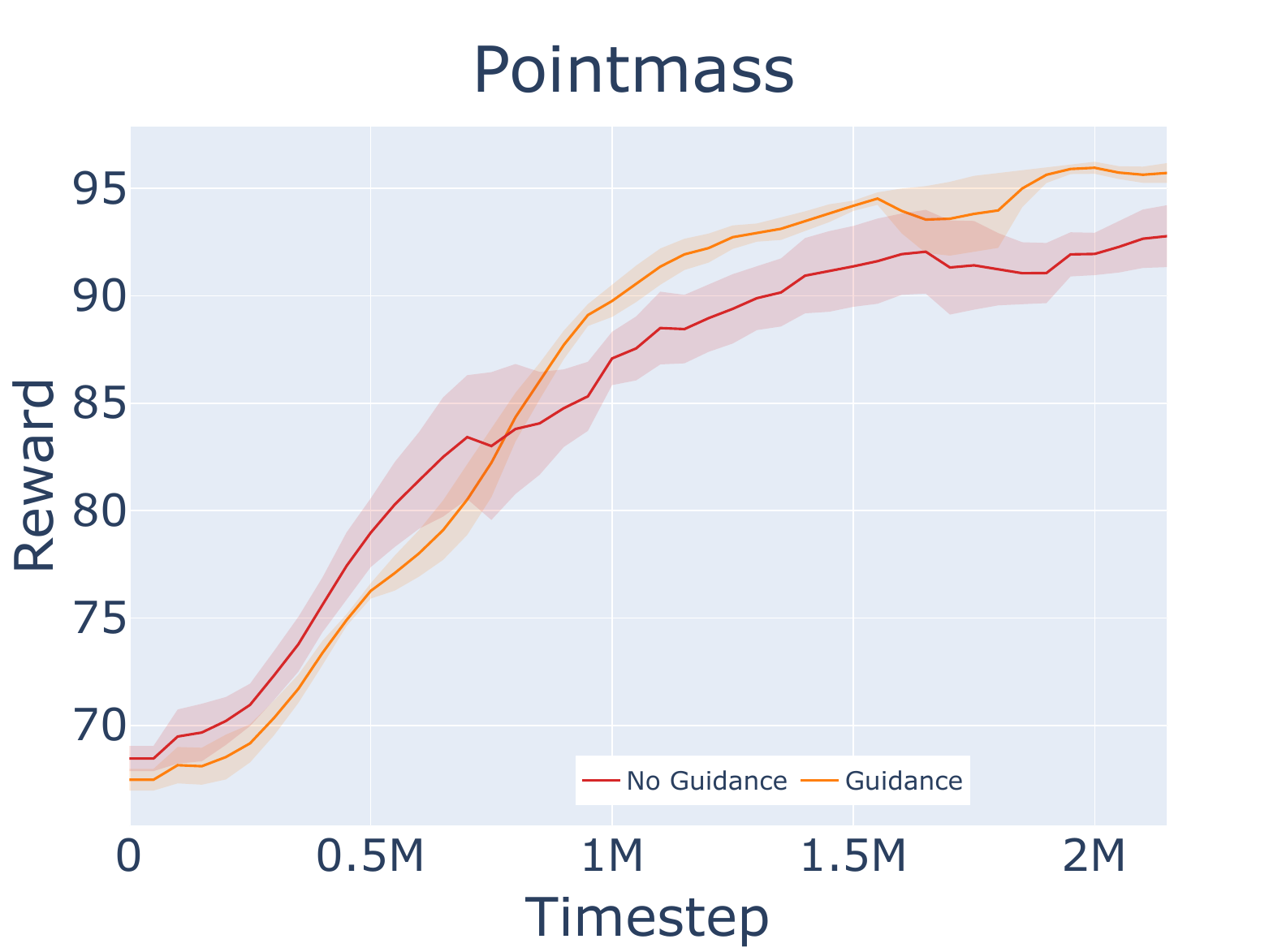} & \hspace{-1.5em}
     \includegraphics[width=0.3\linewidth, trim={0, 0, 0.5cm, 0}, clip]{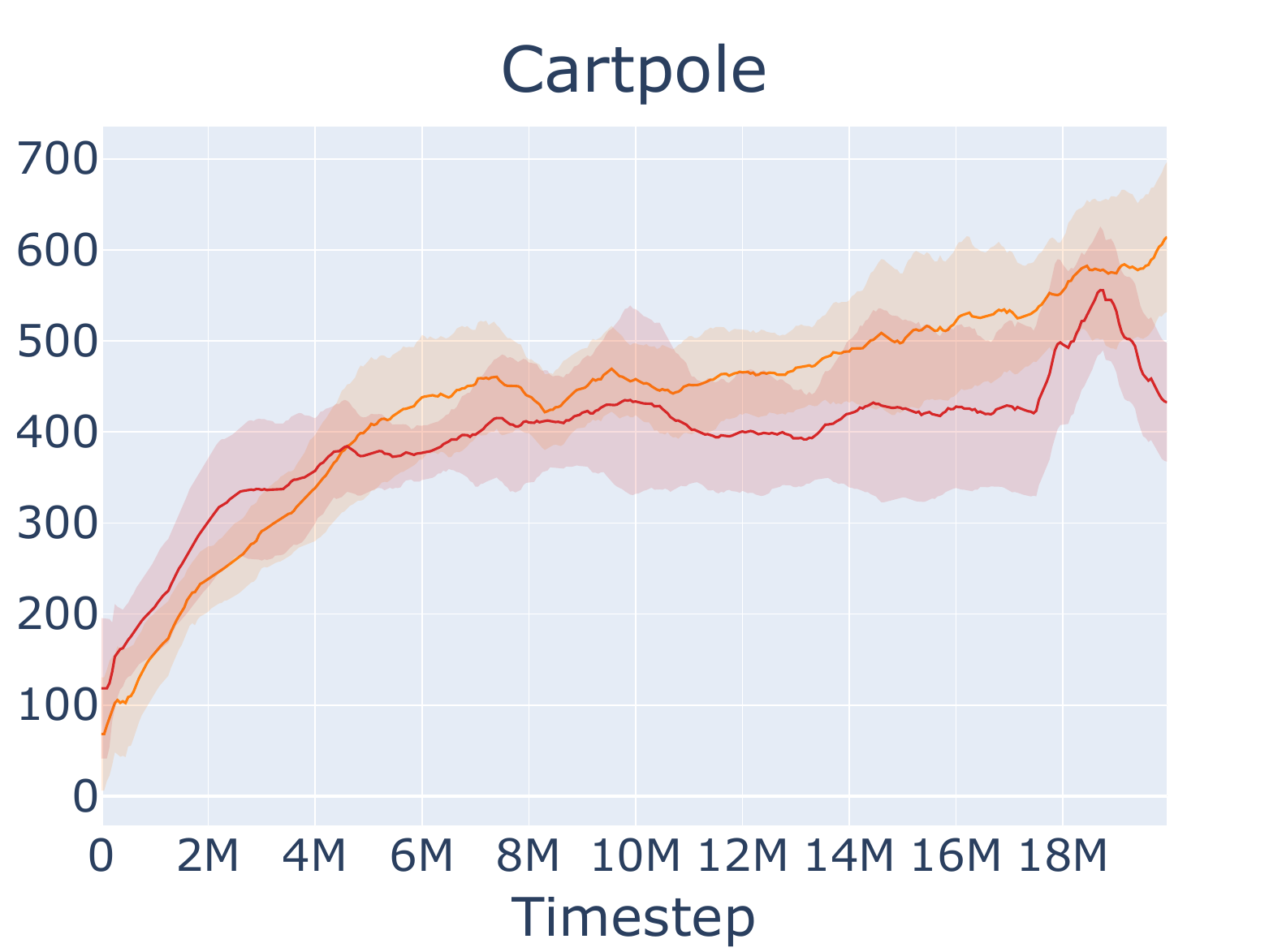} & \hspace{-1.5em}
     \includegraphics[width=0.3\linewidth, trim={0, 0, 0.5cm, 0}, clip]{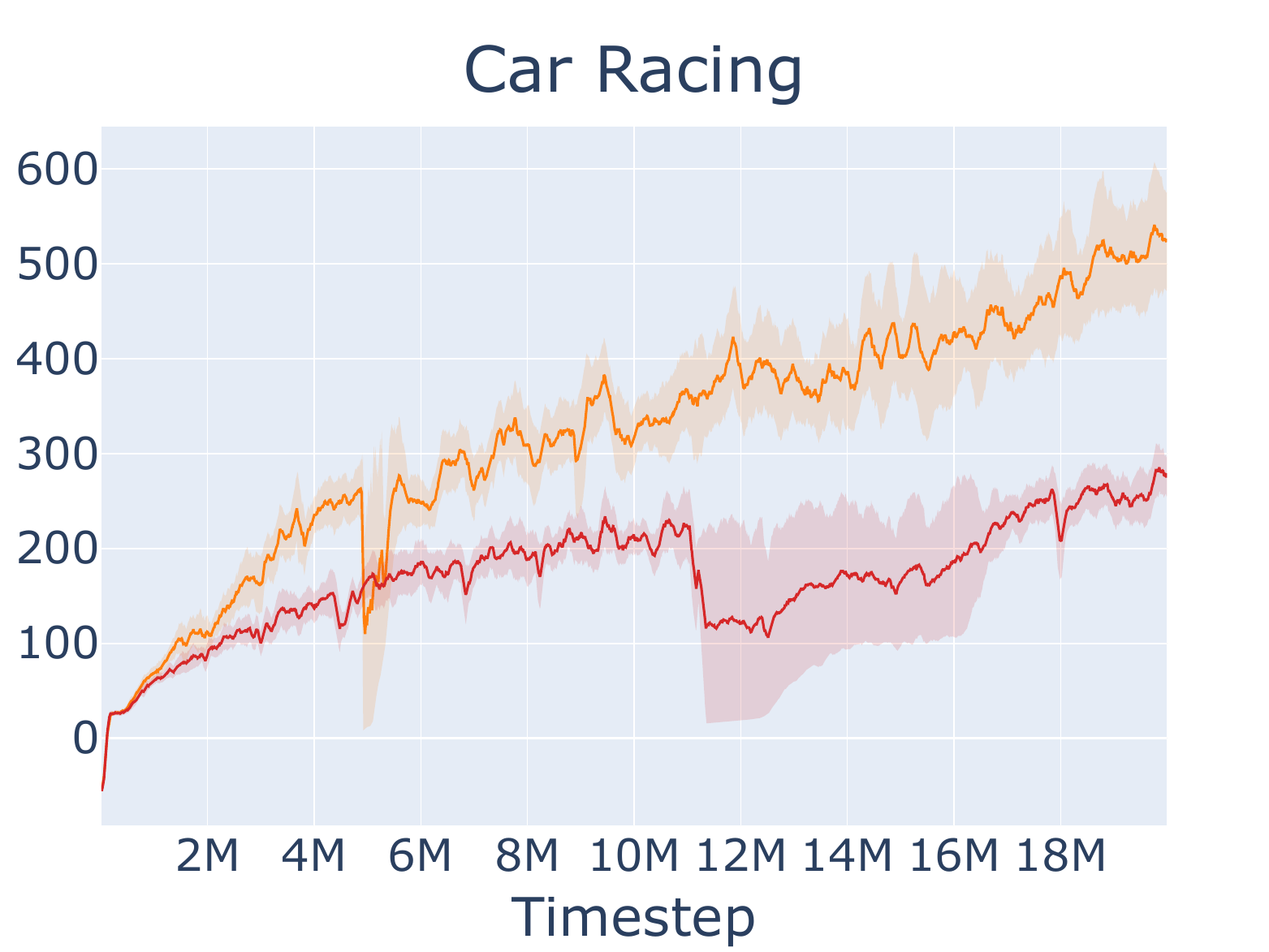}
\end{tabular}
\caption{Ablation on expert guidance across three environments. The benefit of guidance increases with task difficulty.}
\label{fig:guidance_ablation}
\vspace{-1.4em}
\end{figure*}
We compare reward-based policies trained with and without expert guidance to validate this design choice. Fig.~\ref{fig:guidance_ablation} shows that guidance provides a small but consistent improvement in Pointmass and Cartpole, while the benefit is substantially larger in Car Racing.

The difference across environments reflects the role of exploration under limited observability. In Pointmass and Cartpole, suboptimal exploration still yields task reward, so the policy can improve without guidance. In Car Racing, track deviation terminates the episode and eliminates the reward signal. Expert guidance prevents early episode termination during training. This allows the policy to observe sufficient reward signal for effective learning.

\vspace{-0.7em}
\section{Conclusion}
This paper proposes a reward-based policy framework for zero-shot transfer across environments with different observation spaces. We show that a policy conditioned only on reward and action history can be trained with an LSTM architecture and expert guidance, and can serve as a transferable teacher to accelerate observation-based policy learning in a target environment via DAgger, including photorealistic Habitat-Sim reconstructions of real building interiors with a Stretch mobile robot. We further show that reward estimation from target observations is substantially easier than state estimation, which provides an additional practical advantage when rewards are not directly available.

Several directions remain for future work. This work assumes identical transition functions between source and target environments to isolate the observation shift problem. Combining reward-based policies with domain randomization or system identification could extend the framework to settings where both observation and dynamics shift simultaneously.

Finally, this work evaluates reward-based policies in simulation, including photorealistic Habitat-Sim reconstructions of real building interiors; physical hardware deployment remains untested. It is also essential to examine physical observation space adaptations in real-world robotic deployments, such as transitioning mobile platforms from LiDAR to RGB-D sensors or equipping manipulators with extended Field-of-View (FoV) cameras. Prior work has demonstrated that real-world robot learning is feasible with careful experimental design~\cite{kalashnikov2018qtopt, ha2020learning}, and similar setups could be used to validate our framework on physical hardware. In parallel, real-world rewards can often be estimated directly from sensor data without explicit instrumentation: for example, lane deviation in a driving task can be inferred from a camera image. Demonstrating the framework on a physical robotic platform, either with directly measured or estimated rewards, would substantially strengthen its impact for the robotics community.

\vspace{-0.25em}

\section{Acknowledgments}
ChatGPT was used for generating the code for the 3D Car Racing environment and for minor editing. All content was reviewed and validated by the authors.

\vspace{-0.5em}
\bibliographystyle{IEEEtran}
\bibliography{bib.bib}

\vspace{-0.5em}
\appendix

\section{Theoretical Analysis of the Optimality Gap}
\label{sec:appendix_optimality}

We formally bound the optimality gap between observation-based and reward-based policies.

\begin{definition}[$\epsilon$-Reconstructibility]
The state space $S$ is $\epsilon$-reconstructible from $O^{\text{rwd}}$ if there exists a mapping $\phi: O^{\text{rwd}} \mapsto S$ such that for any history $h_t \in O^{\text{rwd}}$:
\begin{equation}
    \Vert s_t - \phi(h_t) \Vert \le \epsilon
\end{equation}
where $\epsilon \ge 0$ is the maximum reconstruction error.
\end{definition}

\vspace{-0.2em}

To relate reconstruction error to policy performance, we assume the value function varies smoothly with the state.

\vspace{-0.2em}

\begin{assumption}[Lipschitz Continuity]
The optimal value function $V^*_{\mathcal{M}}$ is $L$-Lipschitz continuous: for any $s_1, s_2 \in S$,
\begin{equation}
    |V^*_{\mathcal{M}}(s_1) - V^*_{\mathcal{M}}(s_2)| \le L \Vert s_1 - s_2 \Vert
\end{equation}
\end{assumption}

Under these two conditions, the optimality gap of the reward-based policy is bounded linearly by $\epsilon$.

\vspace{-0.2em}

\begin{theorem}[Optimality Bound]
\label{thm:optimality_bound}
The performance gap between the optimal observation-based policy and the optimal reward-based policy satisfies:
\begin{equation}
    V^*_{\mathcal{M}}(s_t) - V^*_{\mathcal{M}^{\text{rwd}}}(h_t) \le \frac{L \epsilon}{1 - \gamma}
\end{equation}
where $\gamma \in [0, 1)$ is the discount factor.
\end{theorem}

\begin{proof}
Define an induced policy $\hat{\pi}(h_t) = \pi^*(\phi(h_t))$, which acts as if the reconstructed state were the true state. By Assumption 1, the single-step value error is bounded by:
\begin{equation}
    |V^*_{\mathcal{M}}(s_t) - V^*_{\mathcal{M}}(\phi(h_t))| \le L \Vert s_t - \phi(h_t) \Vert \le L\epsilon
\end{equation}
Accumulating this error over the infinite horizon under $\hat{\pi}$, and using the fact that the optimal reward-based policy satisfies $V^*_{\mathcal{M}^{\text{rwd}}} \ge V^{\hat{\pi}}_{\mathcal{M}^{\text{rwd}}}$, yields:
\begin{equation}
    V^*_{\mathcal{M}}(s_t) - V^*_{\mathcal{M}^{\text{rwd}}}(h_t) \le \sum_{k=0}^{\infty} \gamma^k L\epsilon = \frac{L \epsilon}{1 - \gamma}
\end{equation}
\end{proof}

\vspace{-0.5em}

The remaining question is how $\epsilon$ scales with $D$. Recovering a $D$-dimensional state from scalar rewards requires at least $D+1$ independent probing steps, and the probability of sufficiently orthogonal measurements decreases exponentially with $D$. We therefore expect $\epsilon(D) = \mathcal{O}(c^D)$ for fixed $K$, predicting an $\mathcal{O}(c^D)$ optimality gap consistent with Table~\ref{table:nd_pointmass}.

\end{document}